\documentclass{article}
\usepackage[T1]{fontenc}
\usepackage{lmodern}
\usepackage{amssymb,amsmath,color}
\usepackage{algorithm}
\usepackage{algorithmic}
\usepackage{url}
\usepackage{graphicx}

\usepackage{todonotes}

\usepackage[mathcal]{eucal}

\newcommand{\BEAS}{\begin{eqnarray*}}
\newcommand{\EEAS}{\end{eqnarray*}}
\newcommand{\BEA}{\begin{eqnarray}}
\newcommand{\EEA}{\end{eqnarray}}
\newcommand{\BEQ}{\begin{equation}}
\newcommand{\EEQ}{\end{equation}}
\newcommand{\BIT}{\begin{itemize}}
\newcommand{\EIT}{\end{itemize}}
\newcommand{\BNUM}{\begin{enumerate}}
\newcommand{\ENUM}{\end{enumerate}}
\newcommand{\BA}{\begin{array}}
\newcommand{\EA}{\end{array}}

\newcommand{\argmin}{\mathop{\rm argmin}}

\newcommand{\rb}{\mathbb{R}}

\newcommand{\BlackBox}{\rule{1.5ex}{1.5ex}}  % end of proof

\newenvironment{proof}{\par\noindent{\bf Proof\ }}{\hfill\BlackBox\\[2mm]}

\newtheorem{proposition}{Proposition}

\newcommand{\mysec}[1]{Section~\ref{sec:#1}}
\newcommand{\eq}[1]{Eq.~(\ref{eq:#1})}
\newcommand{\myfig}[1]{Figure~\ref{fig:#1}}

\newcommand{\lova}{Lov\'asz }

\usepackage[utf8]{inputenc} % allow utf-8 input
\usepackage[T1]{fontenc}    % use 8-bit T1 fonts
\usepackage{hyperref}       % hyperlinks
\usepackage{url}            % simple URL typesetting
\usepackage{booktabs}       % professional-quality tables
\usepackage{amsfonts}       % blackboard math symbols
\usepackage{nicefrac}       % compact symbols for 1/2, etc.
\usepackage{microtype}      % microtypography

\title{Fast Decomposable Submodular Function Minimization using Constrained Total Variation}

\begin{document}
\author{K S Sesh Kumar \\
Department of Computing\\ Imperial College London, UK\\
\texttt{s.karri@imperial.ac.uk}
\and Francis Bach
\\ INRIA and Ecole normale superieure\\
PSL Research University, Paris France.\\
\texttt{francis.bach@inria.fr}
\and Thomas Pock\\
Institute  of  Computer  Graphics  and  Vision,\\
Graz University of Technology, Graz, Austria.\\
\texttt{pock@icg.tugraz.at}
}
\maketitle

\begin{abstract}
We consider the problem of minimizing the sum of submodular set functions assuming minimization oracles of each summand function. Most existing approaches reformulate the problem as the convex minimization of the sum of the corresponding Lov\'asz extensions and the squared Euclidean norm, leading to algorithms requiring total variation oracles of the summand functions; without further assumptions, these more complex oracles require many calls to the simpler minimization oracles often available in practice. In this paper, we consider a modified convex problem requiring  constrained version of the total variation oracles that can be solved with significantly fewer calls to the simple minimization oracles.  We support our claims by showing results on graph cuts for 2D and 3D graphs.
\end{abstract}

\section{Introduction}

A discrete function $F$ defined on a finite ground set $V$ of $n$ objects is said to be
submodular if the marginal cost of each object reduces with the increase in size of the set it is conditioned on, i.e., $F:2^V \to \rb$ is submodular if and only if the marginal cost of an object $\{x\} \in V$ conditioned on the set $A \subseteq V\setminus\{x\}$ denoted by $F(\{x\}|A) = F(\{x\} \cup A) - F(A)$ reduces as the set $A$ becomes bigger.  The diminishing returns property of submodular functions has been central to solving several machine learning problems such as document summarization~\cite{lin2011}, sensor placement\cite{krause11submodularity} and graphcuts~\cite{boykov2001fast}(See~\cite{fot_submod} for more applications). Without loss of generality, we consider normalized submodular functions, i.e., $F(\varnothing) = 0$.

Submodular function minimization (SFM) can be solved exactly using polynomial algorithms but with high computational complexity. One of the standard algorithms is the Fujishighe-Wolfe algorithm~\cite{fujishige2005submodular, NIPS2014_5321} but most recently, SFM were tackled using cutting plane methods~\cite{lee2015faster} and geometric scaling~\cite{Dadush2018}. All the above algorithms rely on value function oracles, that is access to $F(A)$ for arbitrary subsets $A$ of $V$, and solve SFM with high complexities, e.g., $O(n^4\log^{O(1)}n)$ and more. These  algorithms are typically not trivial to implement and do not scale to problems with large ground sets (such as $n = 10^6$ in computer vision applications). For scalable practical solutions, it is imperative to exploit the structural properties of the function to minimize.

Submodularity is closed under addition~\cite{fujishige2005submodular, fot_submod}. We make this structural assumption that is practically useful in many machine learning problems~\cite{stobbe13, komodakis2011mrf} (e.g., when a 2D-grid graph is seen as the concatenation of vertical and horizontal chains) and consider the problem of minimizing a sum of submodular functions~\cite{Kolmogorov2012a}, i.e.,
\BEQ
\min_{A \subset V} F (A)  := \sum_{i = 1}^r F_i(A),
\label{eq:sfm1}
\EEQ
assuming each summand $F_i$, $i=1,\dots,r$, is ``simple'', i.e., with available efficient oracles which are more complex than plain function evaluations. The simplest of these oracles is being able to minimize the submodular function $F_i$ plus some modular function, and we will consider these oracles in this paper. This is weaker than the usual ``total variation'' oracles detailed below.

One of the standard approaches to solve the discrete optimization problem in \eq{sfm1} is to consider an equivalent continuous optimization problem that minimizes the \lova extension~\cite{lovasz1982submodular} $f$ of the submodular function over the $n$-dimensional unit hypercube (see a definition in \mysec{reviewsubmod}). This approach uses a well known result in the submodularity literature that the minima of the set function $F$ and its \lova extension $f$ are exactly the same. More precisely, the continuous optimization problem is given by
\BEQ
\min_{w \in [0,1]^n} f(w) := \sum_{i=1}^r f_i(w),
\label{eq:tv1}
\EEQ
where $f_i$ is the \lova extension of submodular function $F_i$, for each $i \in [r]$. \lova extension of submodular functions are convex but non-smooth and piecewise linear functions. Therefore, we can use subgradients as they can be calculated using greedy algorithms in $O(n\log(n))$ time and $O(n)$ calls to the value function oracle per iteration~\cite{Edmonds}. However, this is slow with $O(1/\sqrt{t})$ convergence rate where $t$ is the number of iterations. Moreover, in signal processing applications, high precision is needed, hence the need for faster algorithms.

An alternative approach is to consider a continuous optimization problem~\cite[Chapter 8]{fot_submod} of the form
\BEQ
\min_{w \in \rb^n}  f(w) + \sum_{j=1}^n \psi(w_j),
\label{eq:tv2}
\EEQ
where $\psi:\rb  \to \rb$ is a convex function whose Fenchel-conjugate~\cite{rockafellar97} is defined everywhere (in order to have a well-defined dual as later shown in \eq{dualj}). This is equivalent to solving all the following discrete optimization problems parameterized by $\alpha \in \rb$,
\BEQ
\min_{A \subset V} F(A) + |A| \psi'(\alpha).
\label{eq:sfm2}
\EEQ

Given the solutions $A^{\alpha}$ for all $\alpha \in \rb$ in \eq{sfm2}, then we may obtain the optimal solution $w^*$ of \eq{tv2} using $w_j^* = \sup(\{\alpha \in \rb, j \in A^{\alpha})$. Conversely, given the optimal solution $w^*$ of \eq{tv2}, we may obtain the solutions $A^{\alpha}$ of the discrete optimizaton problems in \eq{sfm2} by thresholding at $\alpha$, i.e., $\{w^*  \geq \alpha\}$. As a consequence of this we can obtain the solution of \eq{sfm1}, when $\alpha$ is chosen so that $\psi'(\alpha)=0$ (typically $\alpha=0$ because $\psi$ is even). Note that this algorithmic scheme seems wasteful because we take a continuous solution of \eq{tv2} and only keep the signs of its solution. One contribution of the paper is to propose a function $\psi$ that focuses only on values of $w$ close to zero.

\paragraph{Our problem setting and approach.} We assume SFM oracles of the individual summand function, which we refer to as ${\rm SFM\textbf{D}}_i$, for each $i \in [r]$ that gives the optimal solution of
\BEQ
{\rm SFM\textbf{D}}_i: \argmin_{A \subset V} F_i(A) - u^{\top}1_A,
\label{eq:sfmi}
\EEQ
where $u \in \rb^n$ is any $n$-dimensional vector and $1_A \in \{0, 1\}^n$ is the indicator function of the set $A$. Note that the complexity of the oracle does not typically depend on the vector $u$. We consider the following continuous optimization problem, which we refer to as ${\rm SFM\textbf{C}}_i$ for each $i \in [r]$,
\BEQ
{\rm SFM\textbf{C}}_i: \argmin_{w \in \rb^n} f_i(w) -t^{\top}w +  \sum_{j=1}^n \psi(w_j),
\label{eq:tvi}
\EEQ
where $t \in \rb^n$ is any $n$-dimensional vector. In our setting, we consider $\psi$ as the following convex function,
\BEQ
\psi(w) = 
\begin{cases}
\frac{1}{2} w^2 & \text{if } |w| \leqslant \varepsilon, \\
+\infty                   & \text{otherwise,}
\end{cases}
\label{eq:psi}
\EEQ
where $\varepsilon \in \rb_+$. In \mysec{singleFuction}, we show that the continuous optimization problem ${\rm SFM\textbf{C}}_i$ can be optimized using discrete oracles ${\rm SFM\textbf{D}}_i$ using a modified divide-and-conquer algorithm \cite{treesubmod}. In \mysec{multipleFunctions}, we use various optimization algorithms that use ${\rm SFM\textbf{C}}_i$ as the inner loop to the solve the continuous optimization problem in \eq{tv2} consequently solving the SFM problem in \eq{sfm1}.

\paragraph{Related work.} Most of the earlier works have considered quadratic functions for the choice of $\psi$~\cite{treesubmod, seshActiveSet, ene2015random, ene2017decomposable, nishihara2014convergence, chetan12}, i.e., $\psi(v) = \frac{1}{2} v^2$. As a result, ${\rm SFM\textbf{C}}_i$ in \eq{tvi} is referred to as total variation or TV oracle as they solve the problems of the form $\min_{w \in \rb^n} f(w) + \frac{1}{2}(t-w)^2$~\cite{ChambolleMRF}. These oracles are efficient for cut functions defined on chain graphs~\cite{fastTV, barberoTV14} with $O(n)$ complexity. However, this does not hold for general submodular functions. One way to solve continuous optimization problems like ${\rm SFM\textbf{C}}_i$ is to use a sequence of at most $n$ discrete minimization oracles like ${\rm SFM\textbf{D}}_i$ through divide-and-conquer algorithms (see~\mysec{singleFuction}).

%For a quadratic $\psi$, continuous minimization oracles that solve  were assumed by \cite{treesubmod, ene2015random, treesubmod, nishihara2014convergence, chetan12}. While in some subcases (e.g., cuts in one-dimensional chain graphs~\ref{barberoTV14}), these oracles are efficient, they would in general require a sequence of discrete minimization oracles through divide-and-conquer algorithms (see~\mysec{singleFuction}).

Recent work has also focused on using directly discrete minimization oracles of the form ${\rm SFM\textbf{D}}_i$, such as  \cite{seshActiveSet} that considers a total variation problem with active set methods; \cite{stobbe11} used discrete minimization oracles ${\rm SFM\textbf{D}}_i$ but solved a different convex optimization problem. \cite{zhang2018safe} reduces the search space for the SFM problem, i.e., $V$ using heuristics. Our choice of $\psi$ results in a similar reduction of the search space that results in a more efficient solution.

\paragraph{Contributions.}
Our main contribution is to propose a new convex optimization problem that can be used to find the minimum of a sum of submodular set-functions. For graph cuts, this new problem can be seen as a constrained total variation problem that is more efficient that the regular total variation (lesser number of discrete minimization oracle calls). This is beneficial when minimizing the sum of constrained total variation problems, and consequently beneficial for the corresponding discrete minimization problem, i.e., minimizing the sum of submodular functions. For the case of sum of two functions, we show that recent acceleration techniques from~\cite{chambolle2015remark} can be highly beneficial in our case. This is validated using experiments on segmentation of two dimensional images and three dimensional volumetric surfaces. 

Note that we use cuts mainly due to easy access to minimization oracles of cut functions~\cite{kolmogorov04}, but our result applies to all submodular functions.

%The key benefit when applying to sum of functions is that the constrained TV oracle requires fewer oracles than the regular TV oracle, with slightly more calls to the (constrained) TV oracle, and thus with an overall gain in running time. Another improvement is the use of accelerated techniques that bring strong benefits when we consider the sum of two functions. This is validated on experiments on segmentation 2D and 3D images.

\section{Review of Submodular Function Minimization (SFM)}
\label{sec:reviewsubmod}
In this section, we review the relevant concepts from submodular analysis (for more details, see \cite{fot_submod,fujishige2005submodular}). All possible subsets of the ground set $V$ can be considered as the vertices $\{0,1\}^n$ of the hypercube in $n$ dimensions (going from $A \subseteq V$ to $1_A \in \{0,1\}^n$). Thus, any set-function may be seen as a function~$F$ defined on the vertices of the hypercube $\{0,1\}^n$. It turns out that $F$ may be extended to the full hypercube $[0,1]^n$ by piecewise-linear interpolation, and then to the whole vector space $\rb^n$~\cite{fot_submod}.

This extension $f$ is piecewise linear for any set-function~$F$. It turns out that it is convex if and only if $F$ is submodular~\cite{lovasz1982submodular}. Any piecewise linear convex function may be represented as the support function of a certain polytope $K$, i.e., as $f(w) = \max_{s \in K} w^\top s$~\cite{rockafellar97}. For the \lova extension of a submodular function, there is an explicit description of $K$, which we now review.

\paragraph{Base polytope.} We define the \emph{base polytope} as
$
B(F) = 
\big\{ s \in \rb^n, \ s(V) = F(V), \ \forall A \subset V, s(A) \leqslant F(A) \big\},$ where we use the classical notation $s(A) =s^\top 1_A$.
 A key result in submodular analysis is that the \lova extension is the support function of $B(F)$, that is, for any $w \in \rb^n$, 
\BEQ
f(w) = \sup_{s \in B(F)} w^\top s.
\label{eq:lova}
\EEQ
The maximizers  above may be computed in closed form from an ordered level-set representation of~$w$ using a ``greedy algorithm'', which  (a) first sorts the elements of $w$ in decreasing order  such that  $w_{\sigma(1)} \geq \ldots \geq w_{\sigma(n)}$ where $\sigma$ represents the order of the elements in $V$; and (b) computes $s_{\sigma(k)} = F(\{ \sigma(1), \ldots, \sigma(k)\}) - F(\{ \sigma(1), \ldots, \sigma(k-1)\})$. This leads to a closed-form formula for~$f(w)$ and a subgradient.

\paragraph{SFM as a convex optimization problem.}
A key result from submodular analysis~\cite{lovasz1982submodular} is the equivalence between the SFM problem $\min_{A \subseteq V} F(A)$ and the convex optimization problem $\min_{w \in [0,1]^n} f(w)$. One can then obtain an optimal $A$ from level sets of an optimal $w$. Moreover, this leads to the dual problem $\max_{s \in B(F)} \sum_{i=1}^n (s_i)_-$. Note that for our algorithm to work, we need oracles ${\rm SFM\textbf{D}}_i$ that output both the primal variable ($A$ or $w$) and the dual variable $s \in B(F)$.

\paragraph{Convex optimization and its dual.} We consider the continuous optimization problem in \eq{tv2}. Its dual problem derived using \eq{lova} is given by
\BEQ
\label{eq:dualj}
\max_{s \in B(F) } -  \sum_{j=1}^n \psi^\ast(-s_j).
\EEQ
In this paper, we consider the convex function $\psi:\rb \to \rb$ defined in \eq{psi}.
Its Fenchel-conjugate  $\psi^\ast$ is given by
\BEQ
\psi^\ast(s) =
\begin{cases}
\frac{1}{2}s^2   & \text{if }  |s| \leqslant \varepsilon, \\
\varepsilon|s| - \frac{\varepsilon^2}{2} & \text{otherwise.}
\end{cases}
\EEQ

\section{Fast Submodular Function Minimization with Constrained Total Variation}
In this section, we propose an algorithm to optimize the continuous optimization problem in \eq{tv2} using minimization oracles of individual discrete functions ${\rm SFM\textbf{D}}_i$ in \eq{sfmi}. As a first step, we propose a modified divide-and-conquer algorithm to solve the continuous optimization problem ${\rm SFM\textbf{C}}_i$, for each $i \in [r]$ in \mysec{singleFuction}. In \mysec{multipleFunctions}, we use the optimization problems ${\rm SFM\textbf{C}}_i$ as black boxes to solve the continuous optimization problem in \eq{tv2}.
\label{sec:fastsubmod}

\subsection{Single submodular function}
\label{sec:singleFuction}
For brevity, we drop the subscript $i$ and consider the following primal optimization problem.
Algorithm~\ref{Alg:modDivideConquer}  below is an extension of the classical divide-and-conquer algorithm from~\cite{groenevelt1991two}. Note that it requires access to dual certificates for the SFM problems.

\begin{algorithm*}[htb!]
\caption{From ${\rm SFM\textbf{D}}_i$ to ${\rm SFM\textbf{C}}_i$}
\label{Alg:modDivideConquer}
	\begin{algorithmic}[1]
\STATE {\bf Input} Discrete function minimization oracle for $F:2^V \to \rb$ and $\varepsilon \in \rb_+$.\\
\STATE {\bf output} Optimal primal/dual solutions for \eq{tv2} and \eq{dualj} respectively $(w^*, s^*)$ \\
\STATE $A_+ = \argmin_{A \subset V} F(A) + \varepsilon|A|$ with a dual certificate $s_+ \in B(F)$.
\STATE $A_- = \argmin_{A \subset V} F(A) - \varepsilon|A|$ with a dual certificate $s_- \in B(F)$ (we must have $A_+ \subseteq A_-$)
\STATE $w^*(A_+) = -\varepsilon$, $s^*(A_+) = s_+$, $w^*(V \setminus A_-) = \varepsilon$, $s^*(V \setminus A_-) = s_-$\\
\STATE % $U := V \setminus (A_+ \cup A_-)$ 
$ U:= A_- \setminus A_+$ 
and a discrete function $G:2^U$ s.t. $G(B) = F(A_+ \cup B) - F(A_+)$ with \lova extension $g:\{0, 1\}^{|U|} \to \rb$.\\
		\STATE Solve for optimal solutions of $\min_{w \in \rb^{|U|}} g(w) + \frac{1}{2} w^2$ and its dual using divide-and-conquer algorithm~\cite{seshActiveSet} to obtain $(w_U^*, s_U^*)$.  \\
\STATE $(w^*(U), s^*(U) )= (w_U^*, s_U^*)$ \\
\end{algorithmic}
\end{algorithm*}

\begin{proposition}
Algorithm~\ref{Alg:modDivideConquer} gives an optimal primal-dual pair for the optimization problem in \eq{tv2} and \eq{dualj} respectively.
\end{proposition}
\begin{proof}
We claim the following:
\BIT
\item[(a)] $A_+ \subset A_-$.
\item[(b)] $v \in \rb^d$ defined as $v_{A_+}=\varepsilon$, $v_{V \backslash A_-}=-\varepsilon$ and
$v_{A_+ \backslash A_-} = w_{A_+ \backslash A_-}$ 
  is the unique global optimizer of \eq{tv2}.
\item[(c)] $t \in \rb^d$ defined so that $t_{A+} = (s_+)_{A_+}$, $t_{V \backslash A_-} = (s_-)_{{V \backslash A_-}}$ and $t_{A_- \backslash A_+} = s_{A_- \backslash A_+}$, is one of the maximizers of \eq{dualj}.
\EIT

The statement (a) is a consequence of the usual results on minimizing $f(w) + \frac{1}{2}\| w\|_2^2$ and its relationship with SFM. We are going to show (b) and (c) by showing that this is a primal/dual pair with equal objectives.

We need to show that $t \in B(F)$.
We have for any $A \subset V$, using submodularity twice,
\BEAS
t(A) & = &  s_+(A \cap A_+) + t( A\cap ( A_- \backslash A_+) ) + s_-( A \cap (V \backslash A_-)) \\
t(A) & \leqslant &  F(A \cap A_+) + F ( A_+ \cup ( A\cap ( A_- \backslash A_+) ) ) - F(A_+) + s_-( A \cap (V \backslash A_-)) \\
& \leqslant &  F(A \cap A_-) + s_-(  ( A \cup A_-) \backslash A_- )  \\
& \leqslant &  F(A \cap A_-) +  F( A \cup A_-)  - s(A_- )  \\
& = &  F(A \cap A_-) +  F( A \cup A_-)  - F(A_- )  \leqslant   F(A ).
\EEAS
For $A = V$, all inequalities above are equalities, and thus $t \in B(F)$.

We need to show that $w$ is feasible, i.e., that $w_{A_+ \backslash A_-}$ has components in $[-\varepsilon,\varepsilon]$. This is a consequence of classical results for  minimizing $f(w) + \frac{1}{2}\| w\|_2^2$~\cite[Prop.~8.1]{fot_submod}

We can then compute the Lov\'asz extension values exactly and we then have a primal value equal to:
\BEAS
 & & f(v) + \sum_{i=1}^n \psi(v_i) \\
 & = &  f(v) + \frac{1}{2} \| v\|^2 \\
& = & \varepsilon F(A_+) + f_{A_-}^{A_+}( w_{A_+ \backslash A_-}) + \varepsilon F(V) - \varepsilon F(A_-) 
+ \frac{1}{2} \|w_{A_+ \backslash A_-} \|^2
+  \frac{\varepsilon^2}{2} |A_+| + \frac{\varepsilon^2}{2} | V \backslash A_-|.
\EEAS

The dual value is equal to
\BEAS
 & & - \sum_{i=1}^n \psi^\ast(-t_i) \\
& = & - \sum_{i \in A_+ } \psi^\ast(-t_i)
- \sum_{i \in A_- \backslash A_+ } \psi^\ast(-t_i)
- \sum_{i \in V \backslash A_- } \psi^\ast(-t_i) \\
& =  &  
 - \varepsilon \sum_{i \in A_+ } ( |(s_+)_i| - \frac{\varepsilon}{2} ) 
 +
f_{A_-}^{A_+}( w_{A_+ \backslash A_-})  + \frac{1}{2} \|w_{A_+ \backslash A_-} \|^2
 - \varepsilon \sum_{i \in V \backslash A_- } ( |(s_-)_i| - \frac{\varepsilon}{2} )
\\
& =  &  
\varepsilon s_+(A_+) + \frac{\varepsilon^2}{2} |A_+|
 +
f_{A_-}^{A_+}( w_{A_+ \backslash A_-})  + \frac{1}{2} \|w_{A_+ \backslash A_-} \|^2
 +  \varepsilon s_-(V \backslash A_-) + \frac{\varepsilon^2}{2} |V \backslash A_-|,
\EEAS
which is thus equal to the primal value, hence optimality. Here we have used the fact that
$s_+(A_+)+ s_-(V \backslash A_-) = F(V) + F(A_+) - F(A_-)$. Indeed, $s_+$ is the dual certificate for a SFM problem, and has to satisfy $s_+(A+) = F(A_+)$ ~\cite[Prop.~10.3]{fot_submod}. Similarly, $s_-(A_-) = F(A_-)$, which leads to 
$
s_+(A_+)+ s_-(V \backslash A_-) =  s_+(A_+) + s_-(V) - s_-(A_-) = F(A_+) + F(V) -F(A_-)
.$

Note that in the algorithm, there are some free choices for $s_+$ and $s_-$, and that we can take all of them as subvector of the dual to the minimization of $f(w) + \frac{\varepsilon}{2}\| w\|_2^2$, but this is not the only choice.
\end{proof}

Note that the number of steps is at most the number of different values that $w$ may take (the solution $w$ is known to have many equal components~\cite{bach2011shaping}). In the worst case, this is still $n$, but in practice many components are equal to $-\varepsilon$ or $\varepsilon$, thus reducing the number of SFM calls (for $\varepsilon$ very close to zero, only two calls are necessary). In \mysec{experiments}, we show empirically that this is the case, the number of SFM calls decreases significantly when $\varepsilon$ tends to zero.

\subsection{Sum of submodular functions}
\label{sec:multipleFunctions}
In this section, we consider the optimization problem in \eq{tv2} with the function $\psi$ from \eq{psi}. The primal optimization problem is given by
\BEQ
\min_{w \in [-\varepsilon, \varepsilon]^n}   \sum_{i=1}^r f_i(w) + \frac{1}{2}\| w\|_2^2.
\label{eq:tv3}
\EEQ
In order to derive a dual problem with the appropriate structure, we consider the functions $g_i$ defined as follows: 
$
g_i(w) = 
f_i(w) $ if $|w| \leqslant \varepsilon$, and $+\infty$  otherwise, 
with the Fenchel conjugate
$$
g_i^\ast(s_i)
= \sup_{w \in [-\varepsilon,\varepsilon]^n} w^\top s_i - f_i(w)
= \inf_{t_i \in B(F_i)}
\sup_{w \in [-\varepsilon,\varepsilon]^n} w^\top (s_i-t_i)
= \varepsilon  \inf_{t_i \in B(F_i)} \| s_i - t_i \|_1.
$$
Therefore, we can derive the following dual:
\BEA
\nonumber
\min_{w \in [-\varepsilon, \varepsilon]^n}  \sum_{i=1}^r f_i(w) + \frac{1}{2}\| w\|_2^2
& = & \min_{w \in \rb^n}  \sum_{i=1}^r g_i(w) + \frac{1}{2}\| w\|_2^2 \\
\nonumber & = & \min_{w \in \rb^n}   \sum_{i=1}^r \max_{s_i \in \rb^n} \big\{ w^\top s_i - g_i^\ast(s_i) \big\} + \frac{1}{2}\| w\|_2^2 \\
 \label{eq:dual33} & = & \max_{(s_1,\dots,s_r) \in \rb^{n \times r} }    - \sum_{i=1}^r  g_i^\ast(s_i) -  \frac{1}{2}\big\|\sum_{i=1}^r s_i  \big\|_2^2 .
\EEA

We are now faced with the similar optimization problem than previous work~\cite{treesubmod}, where the primal problems is equivalent to computing the proximity operator of the sum of functions $g_1+g_2$. The main difference is that when $\varepsilon$ is infinite (i.e., with no constraints), then the dual functions $g_i^\ast$ are indicator functions of the base polytopes $B(F_i)$, and the dual problem in \eq{dual33} can be seen as finding the distance between two polytopes.

This is not the case for our constrained functions. This limits the choice of algorithms. In this paper, we consider block-coordinate ascent (which was already considered in~\cite{treesubmod}, leading to alternate projection algorithms), and a novel recent accelerated coordinate descent algorithm~\cite{chambolle2015remark}. We could also consider (accelerated) proximal gradient descent on the dual problem in \eq{dual33}, but it was shown empirically to be worse than alternating reflections ~\cite{treesubmod} (which we compare to in experiments, but which we cannot readily extend without adding a new hyperparameter).

\subsection{Optimization algorithms for all $r$}
\label{sec:BCD}
All of our algorithms will rely on the computing the proximity operator of the functions $g_i^\ast$, which we now consider.

\paragraph{Proximity operator.} The key component we will need is the so-called proximal operator of $g_i^\ast$, that is being able to compute efficiently, for a certain $\eta$,
\[
\min_{ s_i \in \rb^n} g_i^\ast(s_i) +  \frac{1}{2\eta} \| s_i - t_i \|_2^2.
\]
Using the classical Moreau identity~\cite{combettes2011proximal}, this is equivalent to solving 
\BEAS
\min_{ s_i \in \rb^n} g_i^\ast(s_i) +  \frac{1}{2\eta} \| s_i - t_i \|_2^2
%& = & 
%\min_{ s_i \in \rb^n} \max_{w_i \in \rb^n} w_i^\top s_i - g_i(w_i) + %\frac{1}{2\eta} \| s_i - t_i \|_2^2
%\\
& = & 
\min_{ s_i \in \rb^n} \max_{w_i \in \rb^n}   w_i^\top s_i - g_i(w_i) + \frac{1}{2\eta} \| s_i \|_2^2
+ \frac{1}{2\eta} \| t_i\|_2^2 - \frac{1}{\eta} s_i^\top t_i
\\
& = &  
\max_{ w_i \in \rb^n}  - g_i(w_i) - \frac{\eta}{2} \| w_i - \frac{1}{\eta} t_i \|_2^2 + \frac{1}{2\eta} \| t_i\|_2^2\\
& = &  
\max_{ w_i \in \rb^n}  - g_i(w_i) + w_i^\top t_i  - \frac{\eta}{2} \| w_i \|_2^2.
\EEAS
This is exactly the oracle ${\rm SFM\textbf{C}}_i$, for which we presented in \mysec{singleFuction} an algorithm using only the discrete oracles ${\rm SFM\textbf{D}}_i$.

%\textbf{Accelerated proximal gradient.} To solve \eq{dual33}, one can use the proximal gradient (i.e., ISTA) technique for the smooth function $(s_1,\dots,s_r) \mapsto \frac{1}{2 } \big\|\sum_{i=1}^r s_i  \big\|_2^2 $, which has a smoothness constant of $r$, with iteration
%\[
%\forall i \in [r], \ s_i^{\rm new} = \argmin_{s_i^{\rm new} \in \rb^n} \ \ 
%g_i^\ast(s_i^{\rm new}) +
%\frac{r}{2} \Big\| s_i^{\rm new}  - s_i
%+ \frac{1}{r} \sum_{j=1}^r s_j
%\Big\|_2^2.
%\]
%This can be accelerated using FISTA, as follows:
%\BEAS
% & & \forall i \in [r], \ s_i^{\rm new} = \argmin_{s_i^{\rm new} \in \rb^n} \ \ 
%g_i^\ast(s_i^{\rm new}) +
%\frac{r}{2} \Big\| s_i^{\rm new}  - t_i
%+ \frac{1}{r} \sum_{j=1}^r t_j
%\Big\|_2^2 \\
%& & \forall i \in [r], t_i^{\rm new} =  s_i^{\rm new}
%+ \beta ( s_i^{\rm new} - s_i ),
%\EEAS
%with $\beta = (k-1)/(k+2)$ at iteration $k$ when $\square^{\rm new}$ is $\square^k$.

\paragraph{Block coordinate ascent.} 
We consider the following iteration 
\[
\forall i \in [r], \ s_i^{\rm new} = \argmin_{s_i^{\rm new} \in \rb^n} \ \ 
g_i^\ast(s_i^{\rm new}) + 
\frac{1}{2} \big\|   \sum_{j=1}^i s_j^{\rm new} + \sum_{j=i+1}^r s_j
\big\|_2^2,
\]
which is exactly block-coordinate ascent on the dual problem in \eq{dual33}. Since the non-smooth function $\sum_{i=1}^r  g_i^\ast(s_i) $ is separable, it is globally convergent, with a convergence rate at least equal to $O(1/t)$, where $t$ is the number of iterations~(see, e.g., \cite{bubeck2015convex}).

\subsection{Acceleration for the special case $r=2$}
\label{sec:acceleration}
When there are only two functions, following~\cite{chambolle2015remark}, the problem in \eq{dual33} can be written as:
\BEQ
\label{eq:new}
\max_{s_1 \in \rb^n} -g_1^\ast(s_1) - h_1(s_1),
\EEQ
with 
\BEAS
h_1(s_1) & = &  \inf_{s_2 \in \rb^n} g_2^\ast(s_2)  + \frac{1}{2 } \| s_1 + s_2 \|^2  \displaystyle =   \sup_{w_2 \in \rb^n} \inf_{s_2 \in \rb^n}  w_2^\top s_2 -  g_2(w_2)  + \frac{1}{2} \| s_1 + s_2 \|^2\\
%& = &  \sup_{w_2 \in \rb^n} \inf_{s_2 \in \rb^n} w_2^\top s_2 -  g_2(w_2)  + \frac{1}{2 } \| s_1\|_2^2 
%+ \frac{1}{2 } \| s_2\|_2^2 
%+  s_2^\top s_1
%\\
& = &  \sup_{w_2 \in \rb^n}     -  g_2(w_2) 
+ \frac{1}{2 } \| s_1\|_2^2 
- \frac{1}{2} \| w_2 +  s_1 \|_2^2, \mbox{ with } s_2 =  w_2 + s_1,
\\
& = &  \sup_{w_2 \in \rb^n}     -  g_2(w_2) 
- \frac{1}{2} \| w_2 \|^2_2 
- w_2^\top s_1 \displaystyle
= \sup_{w_2 \in [-\varepsilon,\varepsilon]^n}     -  f_2(w_2) - w_2^\top s_1
- \frac{1}{2} \| w_2 \|^2_2 .
\EEAS
The function $h_1$ is $1$-smooth with gradient equal to $h_1'(s_1) =  s_1+s_2^\ast(s_1)$. Applying proximal gradient to the problem of maximizing $
\max_{s_1 \in \rb^n} -g_1^\ast(s_1) - h_1(s_1)
$ leads to the iteration
\BEAS
s_2^{\rm new} & = & \argmin_{s_2^{\rm new} \in \rb^n} g_2^\ast(s_2^{\rm new})  + \frac{1}{2 } \| s_1 + s_2^{\rm new} \|^2  \\
s_1^{\rm new} & =  & \argmin_{s_1^{\rm new} \in \rb^n}  g_1^\ast(s_1^{\rm new})
+ \frac{1}{2} \| s_1^{\rm new} - s_1 \|_2^2
+ h_1'(s_1)^\top ( s_1^{\rm new} - s_1 )\\
& =  & \argmin_{s_1^{\rm new}\in \rb^n}  g_1^\ast(s_1^{\rm new})
+ \frac{1}{2} \| s_1^{\rm new} - s_1 \|_2^2
+  ( s_1 + s_2^{\rm new})^\top ( s_1^{\rm new} - s_1 ) \\
& =  & \argmin_{s_1^{\rm new} \in \rb^n}  g_1^\ast(s_1^{\rm new})
+ \frac{1}{2} \| s_1^{\rm new} +
s_2^{\rm new}  \|_2^2
,
\EEAS
which is \emph{exactly} block coordinate descent. Each of these steps are exactly using the same oracle as before. We can now accelerate it using FISTA~\cite{fista} with the step size from the smoothness constant which is equal to $1$. Starting from a pair of iterate $(s_1,t_1)$,
this leads to the iteration:
\BEAS
s_2^{\rm new} & = & \argmin_{s_2^{\rm new} \in \rb^n}  g_2^\ast(s_2^{\rm new}) + \frac{1}{2 } \| t_1 + s_2^{\rm new} \|^2  \\
s_1^{\rm new}
&  = & \argmin_{s_1^{\rm new} \in \rb^n}  g_1^\ast(s_1^{\rm new})
+ \frac{1}{2} \| s_1^{\rm new} +
s_2^{\rm new}  \|_2^2  \\
t_1^{\rm new} & = & s_1^{\rm new}
+ \beta ( s_1^{\rm new} - s_1 )
\EEAS
with $\beta = (t-1)/(t+2)$ at iteration $t$. This algorithm converges in $O(1/t^2)$.

This acceleration can also be used for the case $r>2$ by using the product space trick (see, e.g.,~\cite[Section 3.2]{treesubmod}). However, this requires a correction in the product space that leads to inefficiencies of the algorithm in practice. See~\cite{seshTV} for more details. 

\section{Theoretical Analysis}
\label{sec:analyze}

In this section, we provide a convergence analysis for the methods above. For simplicity of results, we consider the following primal-dual formulation (where both primal and dual variables live in bounded sets):
\BEA
\nonumber
\min_{w \in [-\varepsilon, \varepsilon]^n}  \sum_{i=1}^r f_i(w) + \frac{1}{2}\| w\|_2^2
\nonumber & = & \min_{w \in \rb^n}    \sum_{i=1}^r \max_{t_i \in B(F_i)}  w^\top t_i  +\sum_{j=1}^n \psi(w_j)  \\
 \label{eq:dual3} & = & \max_{(t_1,\dots,t_r) \in B(F_1) \times \cdots \times B(F_r) }     - \sum_{j=1}^n 
 \psi^\ast \big(
 \sum_{i=1}^r t_{ij}
 \big).
\EEA
We assume that we have a pair $(w,t_1,\dots,t_r)$ of approximate primal-dual solutions for \eq{dual3}, with a duality gap $\eta_{\rm C}$. This leads to a pair $(w,u)$ of primal-dual approximate solutions for
\BEQ
\label{eq:primaldual_certif}
\min_{w \in [-\varepsilon, \varepsilon]^n}   f(w) + \frac{1}{2}\| w\|_2^2 \displaystyle = \max_{u \in B(F)}     - \sum_{j=1}^n 
 \psi^\ast (u_j),
\EEQ
for which  we can get an approximate subset of $V$.
\begin{proposition}
\label{prop:certif}
Given a feasible primal candidate $w$ for \eq{primaldual_certif} with suboptimality $\eta_{\rm C}$, one of the suplevel sets $\{ w \geqslant \alpha\}$  of $w$ is an $\eta_{\rm D}$-optimal minimizer of $F$,
with $\eta_{\rm D} = \frac{\eta_{\rm C}}{4 \varepsilon} + \sqrt{ \frac{\eta_{\rm C} n}{2} }$.
\end{proposition}
\begin{proof}
We follow the proof of~\cite[Prop.~10.5]{fot_submod}, which corresponds to the case $\varepsilon = +\infty$.

From a feasible primal candidate, we can always build a dual candidate $s$ (e.g., by taking any dual maximizer). If we assume that for all $\alpha \in [-c,c]$, for $c \in [0,\varepsilon]$ we have
$(F+\psi'(\alpha))(\{ w \geqslant \alpha\}) - (s+\psi'(\alpha))_-(V) > \eta_{\rm C} / (2 c)$, then we obtain that
\[
\eta_{\rm C}
\geqslant \int_{-c}^c 
(F+\psi'(\alpha))(\{ w \geqslant \alpha\}) - (s+\psi'(\alpha))_-(V) d\alpha > \eta_{\rm C},
\]
which is a contradiction. Thus, we must at least one $\alpha \in [-c,c]$ such that 
$(F+\psi'(\alpha))(\{ w \geqslant \alpha\}) - (s+\psi'(\alpha))_-(V) \leqslant \eta_{\rm C} / (2 c)$.
This implies that
\[
F(\{ w \geqslant \alpha\}) -  s+_-(V)
\leqslant \eta_{\rm C} / (2 c) + c n.
\]
This means that at least one level set of $w$ has a certified gap less than 
\BEAS
\eta_{\rm D} & = &  \inf_{ c\in [0,\varepsilon]} 
\eta_{\rm C} / (2 c) + c n 
\\
& = &  \inf_{ c\in [0,1]} 
\eta_{\rm C} / (2 \varepsilon c) + c n \varepsilon 
\\
& = &   ( 2 n \varepsilon) \times
\inf_{ c\in [0,1]} \frac{1}{2}( c + \frac{1}{c} \frac{\eta_{\rm C}}{4 n \varepsilon^2} )  \\
& \leqslant &  ( 2 n \varepsilon)  \times \big(
\sqrt{\frac{\eta_{\rm C}}{4 n \varepsilon^2} } + \frac{1}{2}  \frac{\eta_{\rm C}}{4 n \varepsilon^2}
\big)
= \sqrt{ \eta_{\rm C} n / 2} + \eta_{\rm C} / ( 4 \varepsilon)
\EEAS
using
$ \inf_{ c\in [0,1]} \frac{1}{2}( c + \frac{a}{c} ) \leqslant \sqrt{a} + a/2$.

\end{proof}

%If we assume that we solve the problem above up to $\eta_{\rm C}$ (with a primal candidate), then we solve the SFM problem with  $\frac{\eta_{\rm C}}{2} + \frac{n \varepsilon}{4}$ for SFM. \note{FB: more details here and below.}

%A more careful study following the usual proof of SFM gap from TV gaps, leads to, for any $c \in (0,1)$, to a gap among the sorted values of $w$ of at most $\frac{\eta_{\rm C}}{2c} + \frac{c\varepsilon n}{2}$. Thus, the gap is less than  $\eta_{\rm D} = \inf_{c \in (0,1)}  \frac{\eta_{\rm C}}{2c} + \frac{c\varepsilon n}{2}$, which is equal to $\sqrt{ \eta_{\rm C} n \varepsilon }$ when $\eta_{\rm C}\leqslant n \varepsilon$, and which we can summarize as $\eta_{\rm C} + \sqrt{ \eta_{\rm C} n \varepsilon}$.

%When $\varepsilon$ is large, we recover the usual non-constrained total variation, and the dependence of the gap $\eta_{\rm D}$ is proportional to $\sqrt{\eta_{\rm C}}$, while when $\varepsilon$ is small, it is proportional to ${\eta_{\rm C}}$, which is a better scaling.

Since our dual problems are all $O(1)$-smooth (using the traditional definitions of smoothness~\cite{nesterov2013introductory}), their guarantees will always be of the form $\eta_{\rm C} = \frac{\Delta^2}{ t^\alpha}$ where $\Delta$ is a notion of diameter of the base polytopes and $\alpha=2$ for accelerated algorithms and $\alpha=1$ for  plain algorithms. 
The overall discrete gap is thus up to constant terms
\[
{\eta_{\rm D}} = \frac{\Delta \sqrt{n}}{t^{\alpha/2}}
+ \frac{\Delta^2}{\varepsilon t^\alpha}.
\]
We see clearly that the final bound on the (discrete) gap is decreasing with $\varepsilon$. This suggest to use $\varepsilon$ proportional to $\frac{\Delta}{\sqrt{n t^\alpha}}$ to take it as small as possible while only losing a factor of 2 in the convergence bound.

\paragraph{Guarantees for FISTA applied to the dual of \eq{dual3}.}
The function $\psi^\ast$ is $O(1)$-smooth, and the objective in \eq{dual3} is $r$-smooth. Each $B(F_i)$ has a square diameter less than
less $\Delta_i^2 = \sum_{j=1}^n \big[ F_i(\{j\}) + F_i( V \backslash \{j\} ) - F_i(V) \big]^2$. Thus, in the result above, we have
$\Delta^2 = r \sum_{i=1}^r \Delta_i^2$ and $\alpha=2$. Owing to~\cite[Cor. 2(b)]{tseng}, these guarantees extend to the corresponding primal iterate $w$.

\paragraph{Guarantees for primal-dual algorithms applied to \eq{dual3}.}
We consider the primal-dual formulation
\[
\min_{w \in [-\varepsilon, \varepsilon]^n} 
\max_{(t_1,\dots,t_r) \in B(F_1) \times \cdots \times B(F_r) }
 w^\top \big( \sum_{i=1}^r t_i \big) + \sum_{j=1}^n \psi(w_j).
\]
The primal set has squared diameter $n \varepsilon^2$; the dual set has squared diameter less than $\sum_{i=1}^r \Delta_i^2$, the bilinear function has a largest singular value equal to $\sqrt{r}$. Thus, from~\cite{chambolle2016ergodic}, we get a guarantee from a primal-dual algorithm, of the form
$\Delta^2 = r \sum_{i=1}^r \Delta_i^2 + \varepsilon \sqrt{nr}   \sqrt{\sum_{i=1}^r \Delta_i^2}$. We thus get overall
a guarantee of the same form as above, with the same dependency in $\varepsilon$.
\begin{figure}[t]
\centering
\begin{tabular}{ccc}
\includegraphics[width=0.3\textwidth]{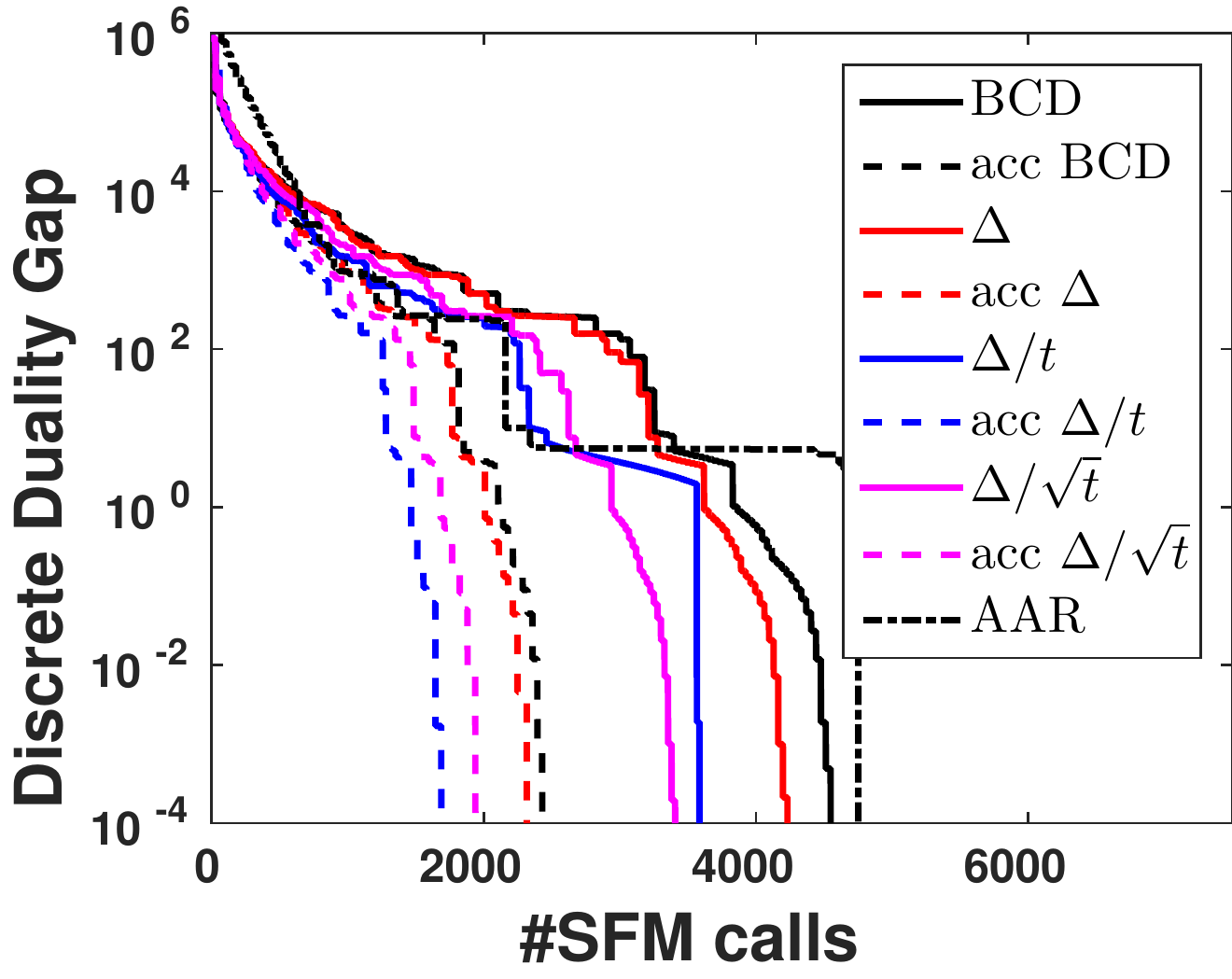} &
\includegraphics[width=0.3\textwidth]{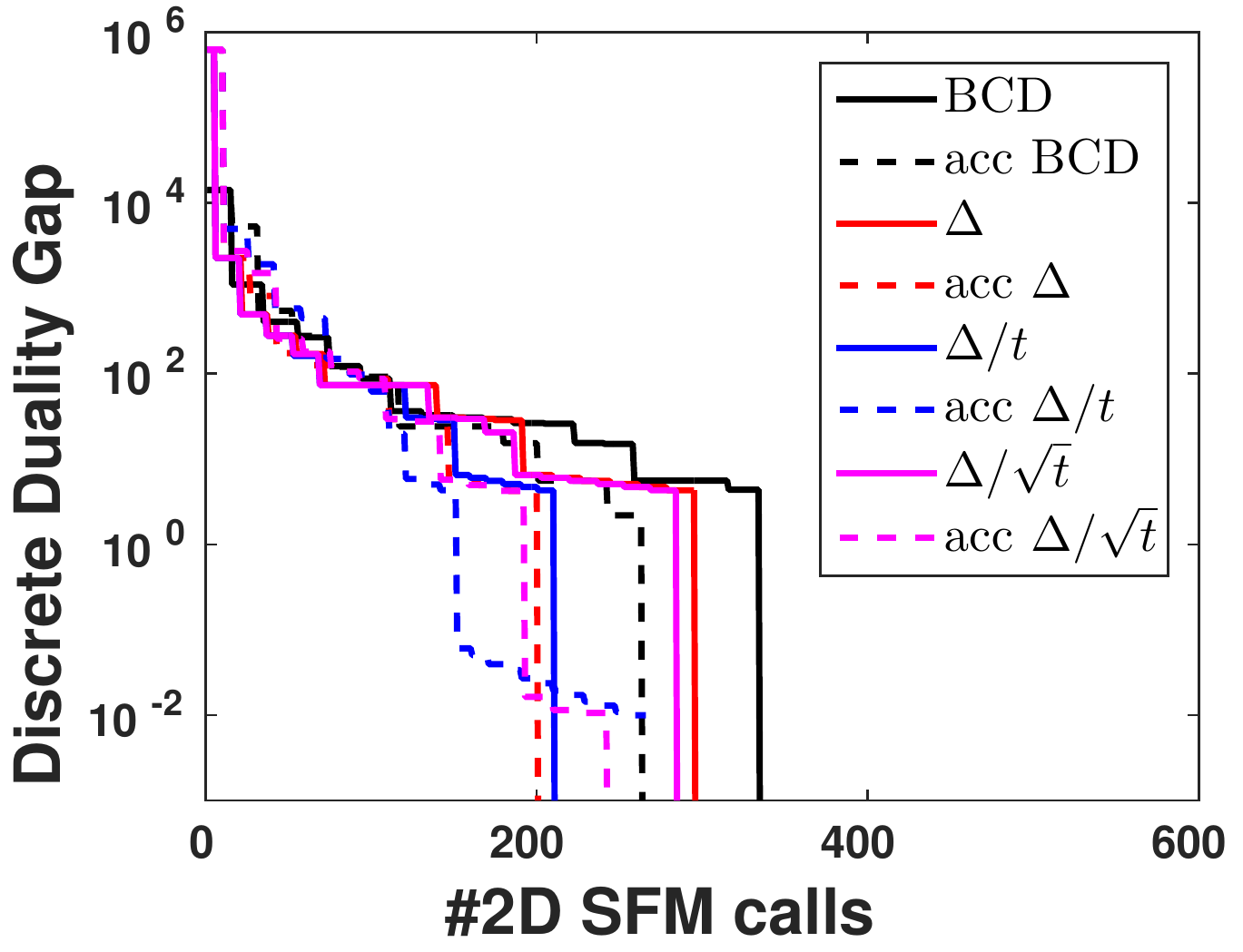} &
\includegraphics[width=0.3\textwidth]{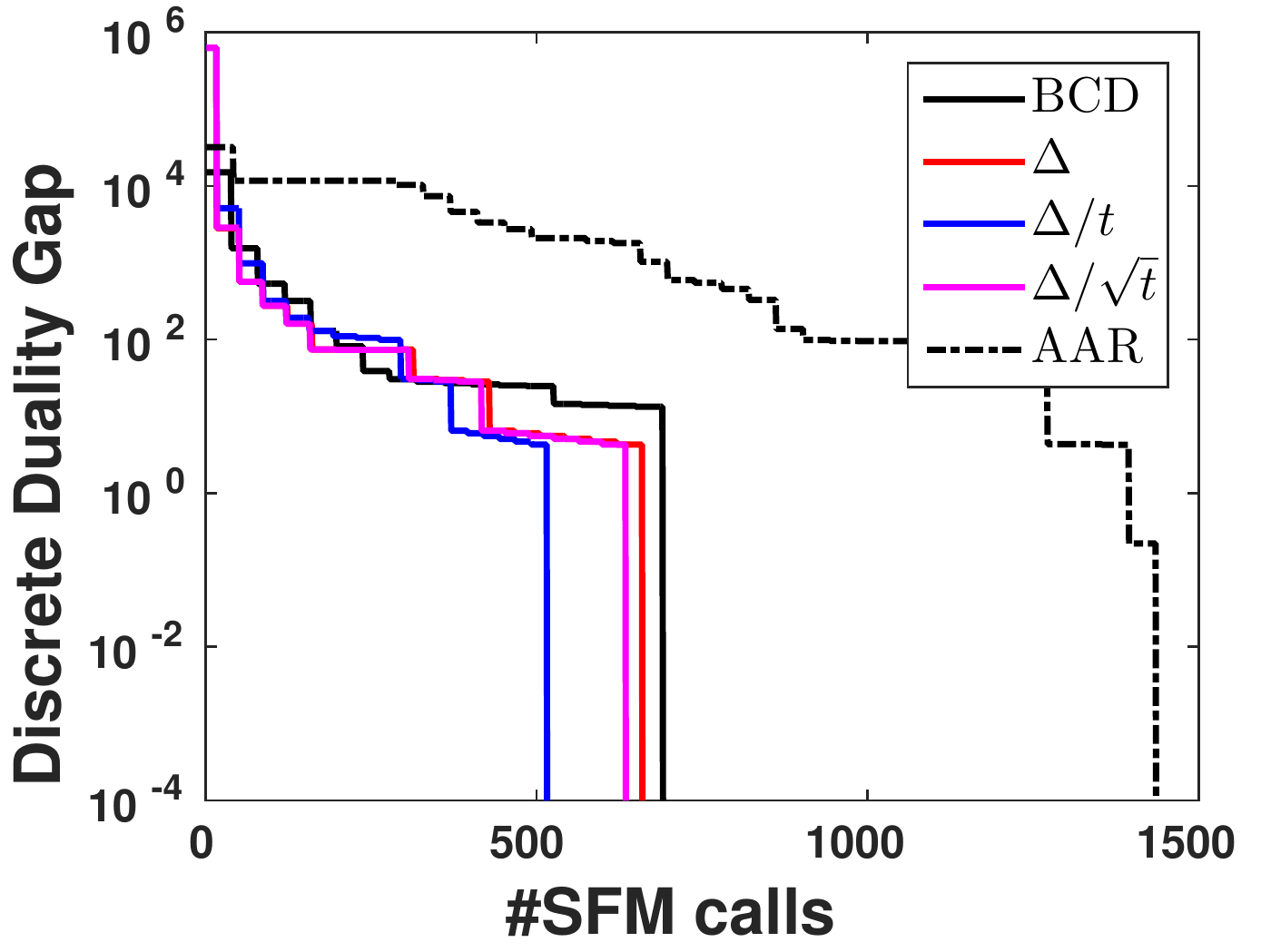} \\
(a) & (b) & (c)
\end{tabular}
\caption{Comparison to state-of-the-art algorithms for 2D and 3D SFM.}
\label{fig:SFM}
\end{figure}

\section{Experiments}
\label{sec:experiments}
In this section, we consider the minimization of cut functions~\cite{boykov2001fast} that are an important examples of submodular functions. In our experiments, we consider the problem of minimizing cuts on 2D images and 3D volumetric surfaces for segmentation. We consider a two-dimensional image of size $n = 2400 \times 2400 = 5.8 \times 10^6$ pixels,  and a 3D volumetric surface of size $n = 102 \times 100 \times 79 = 8.1 \times 10^5$ voxels. The SFM oracles are obtained by using max-flow codes, which is the dual of the min-cut problem. We compare our results to the standard block coordinate descent  (\texttt{BCD})~\cite{treesubmod} and averaged alternating reflections algorithm (\texttt{AAR})~\cite{treesubmod}, which are using full total variation oracles (which we solve using the usual divide-and-conquer algorithm that is only using SFM calls). 

In our approach, we have a  parameter $\varepsilon$ dependent on $\frac{\Delta}{\sqrt{n t^\alpha}}$, where $\Delta$ is the notion of diameter of the base polytope, $n$ is the number of elements in the ground set and $t$ is the number of iterations. In our experiments, we choose $\varepsilon$ proportional to \texttt{$\Delta$, $\Delta / t$} and \texttt{$\Delta / \sqrt{t}$} and respectively represent them by the same terms in \myfig{SFM}. For the case of the sum of two functions, the block coordinate descent can be accelerated~\cite{chambolle2015remark} as shown in \mysec{acceleration}. We refer to their accelerated versions as \texttt{acc BCD, acc $\Delta$, acc $\Delta / t$} and \texttt{acc $\Delta / \sqrt{t}$} respectively. Therefore, \texttt{BCD, acc BCD, AAR} use quadratic $\psi$ and the rest use $\psi$ as defined in \eq{psi}.

\myfig{SFM} shows the performance of various algorithms on different problems which we detail below. The horizontal axes represents the number of discrete minimization oracles, i.e., ${\rm SFM\textbf{D}}_i$ required to solve the SFM and the vertical axes represents the discrete duality gap given by
\[
{\rm gap}(A,s) = F(A) - s_{-}(V),\]
where $A \subset V$, $s \in B(F)$ are the discrete primal-dual pairs and $s_{-}(V) = \sum_{i=1}^n \min(s_i, 0)$. We consider three experiments that may be broadly classified into sum of two functions and sum of three functions.

\paragraph{Sum of two functions ($r=2$).} In this case, we consider minimization of the submodular function that can be written as sum of two submodular functions, i.e., $F = F_1 + F_2$. We consider the problem of mininiming graph cuts on 2D grid that can be written as the sum of horizontal and vertical chain graphs in \myfig{SFM}-(a). In this case, the ${\rm SFM\textbf{D}}_i$ orcale represents the min-cut on a chain  graph while the SFM problem represents min-cut on a 2D grid. We can observe that the constrained total variation formulation reduces the number of min-cut/ max-flow calls when compared to full total variation. Here, we explicitly calculate the diameter of the base polytope $\Delta$ for choosing $\varepsilon$.

\begin{figure}[htb]
\centering
\begin{tabular}{ccc}
\includegraphics[width=0.3\textwidth]{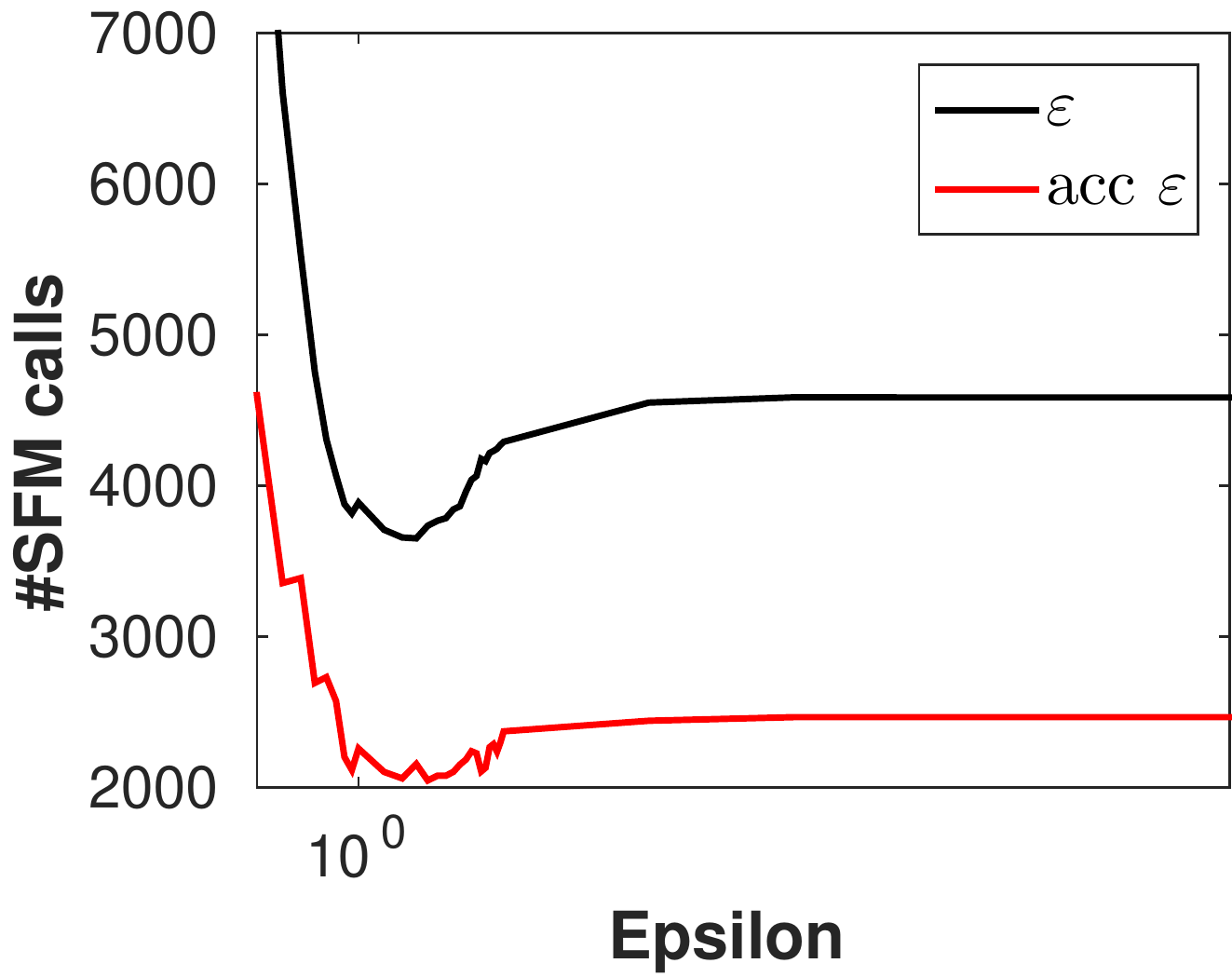} &
\includegraphics[width=0.3\textwidth]{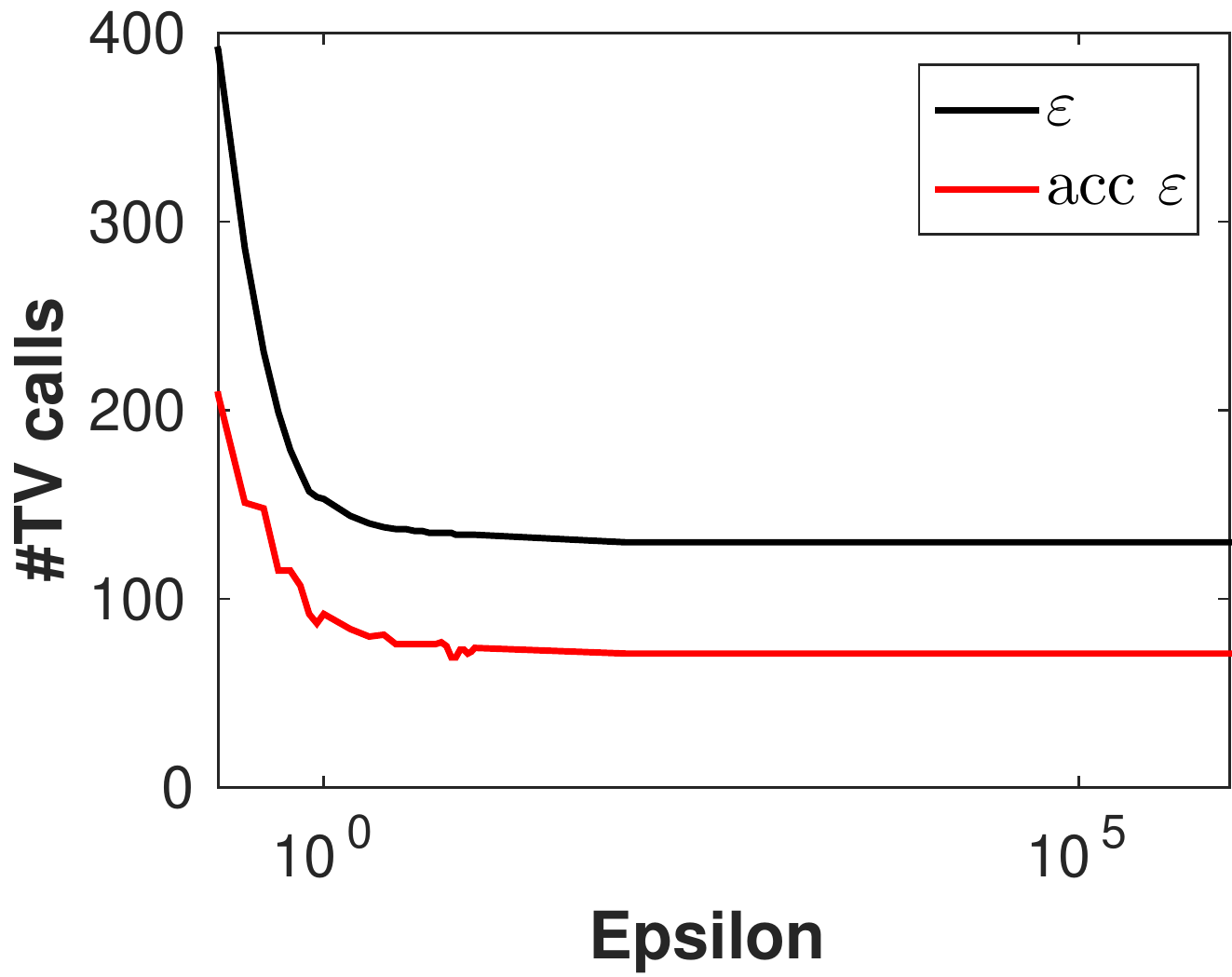} &
\includegraphics[width=0.3\textwidth]{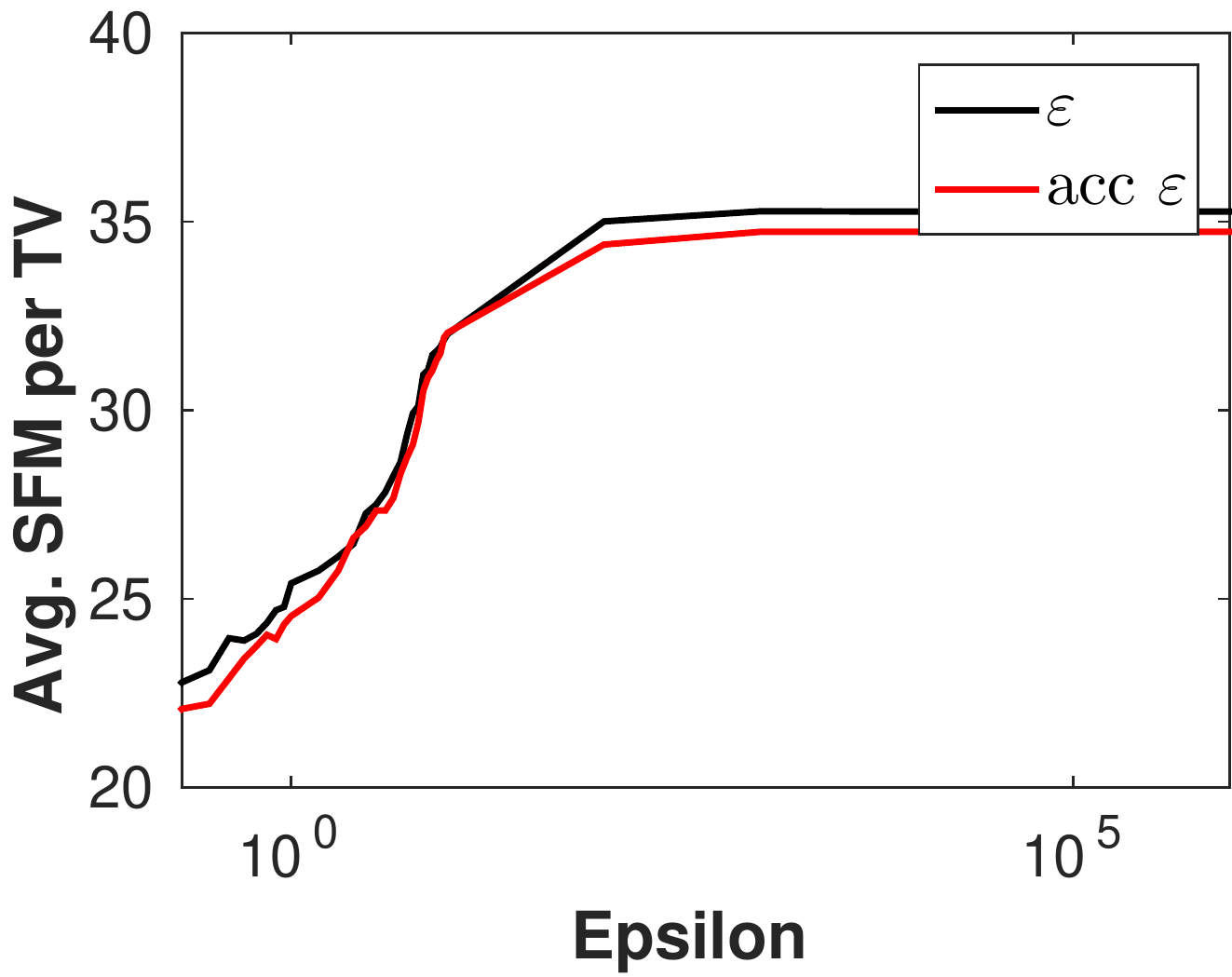}\\
(a) & (b) & (c) \\
\includegraphics[width=0.3\textwidth]{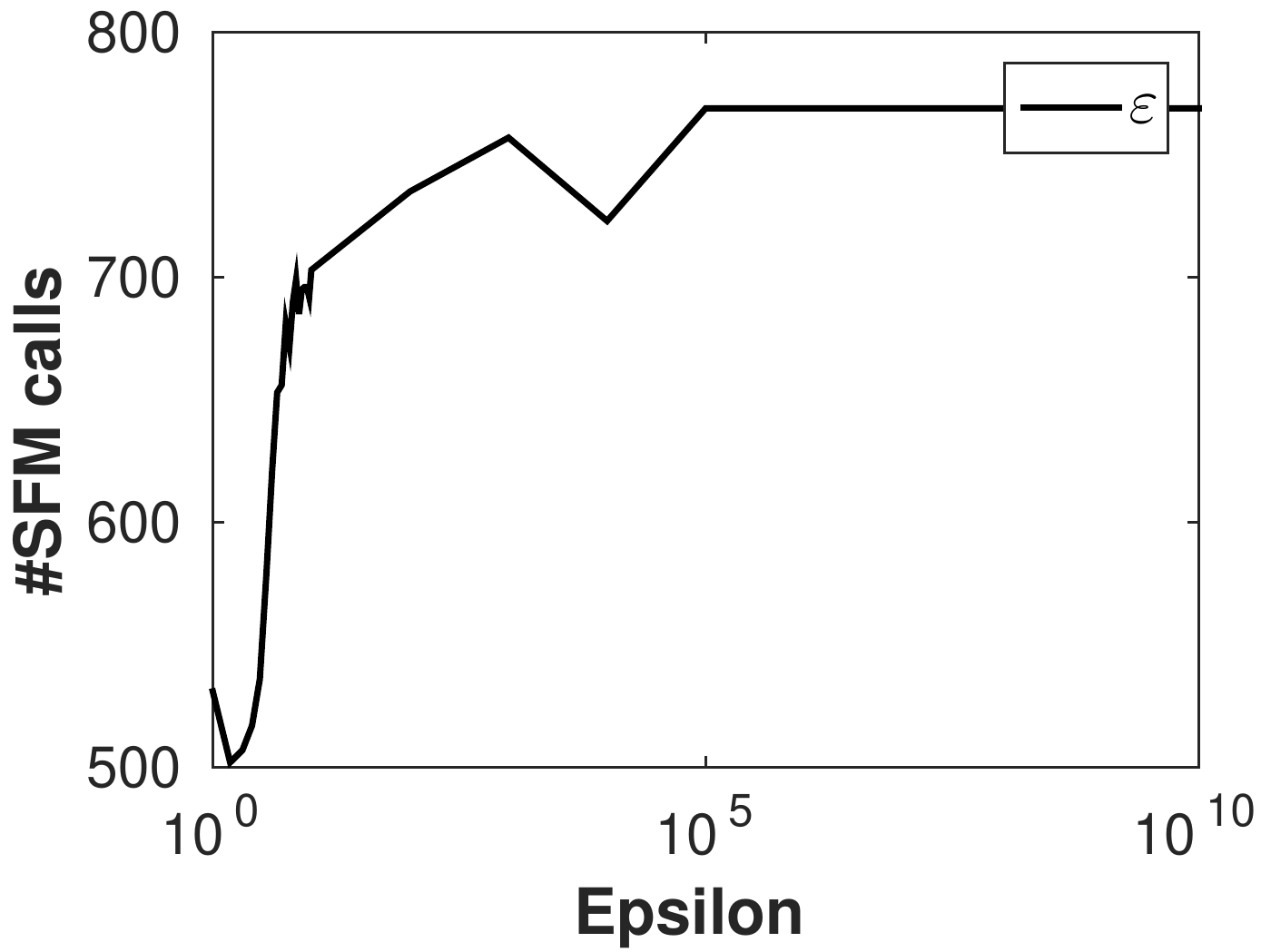} &
\includegraphics[width=0.3\textwidth]{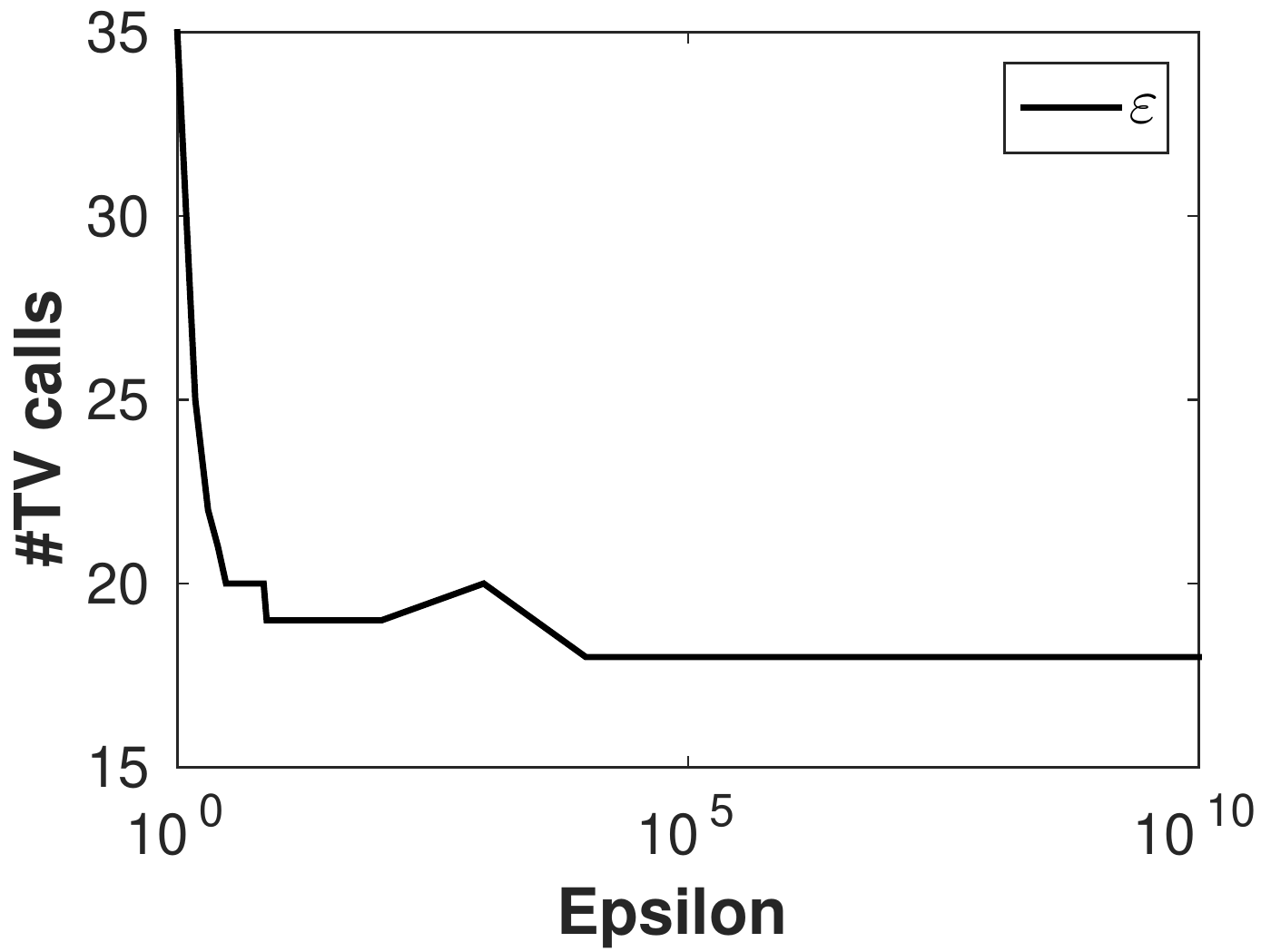} &
\includegraphics[width=0.3\textwidth]{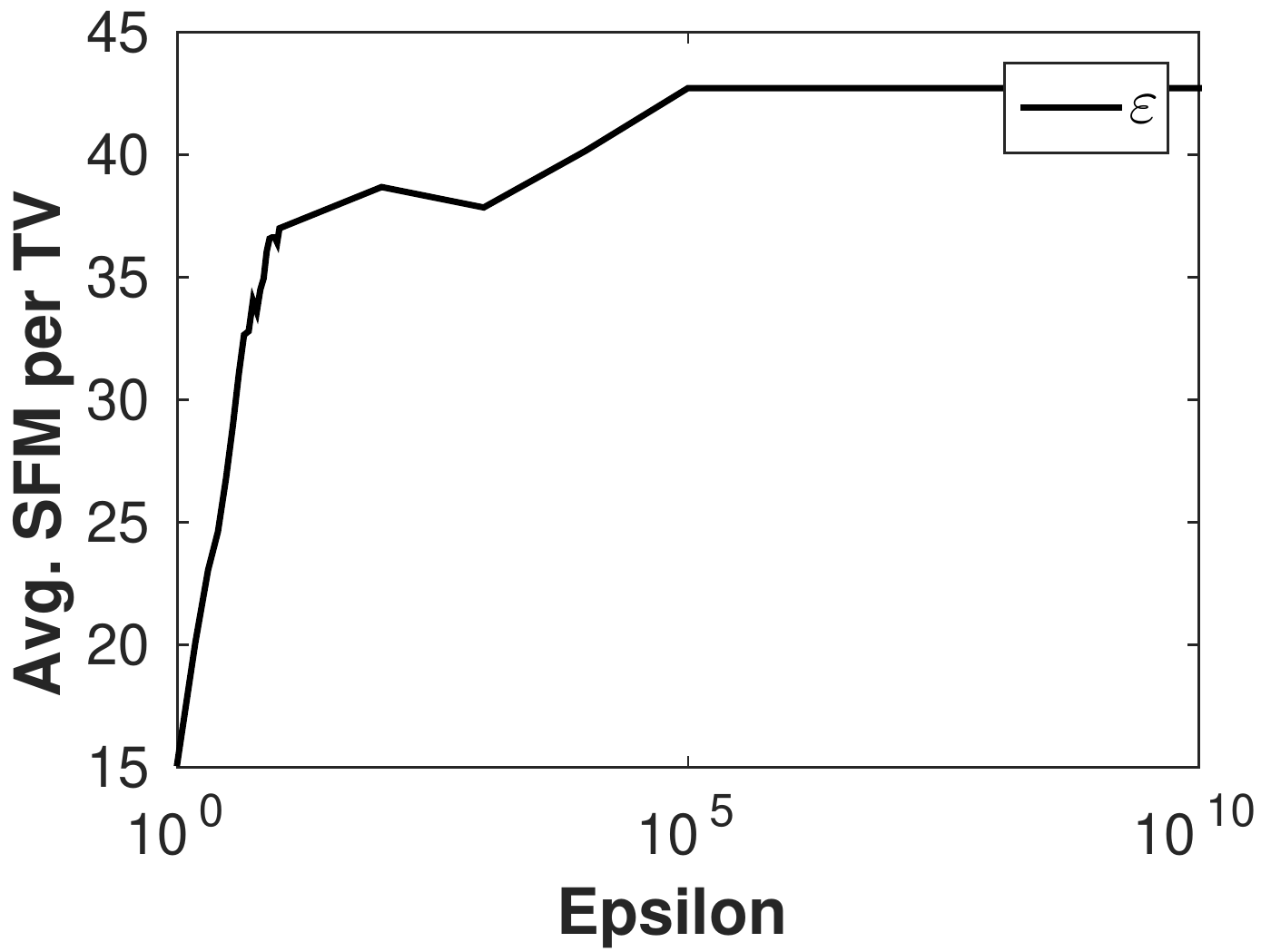}\\
(d) & (e) & (f)
\end{tabular}
\caption{\texttt{BCD} and \texttt{acc BCD} for  2D function: $F = F_1 + F_2$ and 3D functions: $F = F_1 + F_2 + F_3$.}
\label{fig:epsilon}
\end{figure}

\myfig{epsilon}-(a) shows the total number of ${\rm SFM\textbf{D}}_i$ oracle calls required to solve the SFM problem for different values of $\varepsilon$. \myfig{epsilon}-(b) shows the total number of constrained TV ${\rm SFM\textbf{C}}_i$ calls required to solve the SFM problem. \myfig{epsilon}-(c) shows the average number of ${\rm SFM\textbf{D}}_i$ oracle calls required to solve a single ${\rm SFM\textbf{C}}_i$ problem. The algorithms considered in these graphs are  BCD and accelerated BCD algorithms using constrained total variation represented by \texttt{$\varepsilon$} and \texttt{acc $\varepsilon$} respectively. We clearly see the trade-off for the choice of  $\varepsilon$: the number of SFM calls per TV calls increases with $\varepsilon$, while the number of TV calls decreases, leading to intermediate values of $\varepsilon$ which lead to significant gains in the total number of SFM calls in \myfig{epsilon}-(a).

We consider 3D grid that can be decomposed into 2D frames and chains graphs. In \myfig{SFM}-(b), we show the performance of our algorithm compared to other state-of-the-art algorithms for this decomposition. In this case we use two different discrete oracles ${\rm SFM\textbf{D}}_i$, i.e., min-cut on a chain and min-cut on a 2D grid to solve SFM, i.e., min-cut on the 3D grid. We show only the number of oracle calls to min-cut on 2D grids for analysis as they are more expensive than min-cuts on chains. 

\paragraph{Sum of three functions ($r=3$).} In this case, we consider minimization of the submodular function that can be written as sum of three submodular functions, i.e., $F = F_1 + F_2 + F_3$. Min-cut on the 3D grid can also be seen as sum of chain graphs in three directions, thereby using discrete minimization oracles only of the chain graphs. \myfig{SFM}-(c) shows the number of calls to 1D min-cut to solve the 3D min-cut problem using various continuous optimization problems and algorithms. Our approach considerably reduces the number of calls to 1D min-cut (${\rm SFM\textbf{D}}_i$) oracles. \myfig{epsilon}-(d) shows the total number of 1D min-cuts (${\rm SFM\textbf{D}}_i$) to solve 3D SFM for various values of $\varepsilon$. \myfig{epsilon}-(e) shows the total number of constrained total variation ${\rm SFM\textbf{C}}_i$ calls required to solve the SFM problem. \myfig{epsilon}-(e) shows the average number of ${\rm SFM\textbf{D}}_i$ oracle calls required to solve ${\rm SFM\textbf{C}}_i$ for this problem. We observe a similar behavior than for $r=2$, with best values of $\varepsilon$ not being very small or very large.

%\BIT
%\item For a given problem, study the impact of $\varepsilon$. When it goes to zero and when it goes to infinity, by studying the number of constrained TV calls (decreasing in $\varepsilon$, exploding at zero), the number of SFM calls per TV calls (increasing in $\varepsilon$ if no warm start is done, equal to one for $\varepsilon=0$), and the overall number of SFM calls (with a sweet spot in $\varepsilon$).
%\item Given a fixed strategy for $\varepsilon$, compare our various algorithms for $r=2$ and $r>2$.
%\item Compare to the state of the art in terms of number of SFM calls.
%\EIT

%\begin{figure}[htb]
%\centering
%\begin{tabular}{cc}

%\end{tabular}
%\caption{Comparison to state of art algorithms for both 2D and 3D functions.}
%\label{fig:2D3D}
%\end{figure}

\section{Conclusion}
\label{sec:conclusion}

In this paper, we have proposed a simple modification of state-of-the-art algorithms for decomposable submodular function minimization. Adding box constraints to the continuous optimization problems allow for significant reduction in the number of individual submodular function minimization calls. The application of accelerated block coordinate ascent techniques makes the speed-up stronger. These techniques are easily parallelizable and it would be interesting to compare to dedicated parallel algorithms for graph cuts~\cite{shekhovtsov2011distributed}. Moreover, these speed-ups could be extended to more general submodular optimization problems~\cite{bach2016submodular}.

\paragraph{Acknowledgments.}
This research was partially funded by the Leverhulme Centre for the Future of Intelligence and the Data Science Institute at Imperial College London. We acknowledge support from the European Research Council (SEQUOIA project 724063) and (HOMOVIS project 640156).

\bibliography{super_tree.bib}

\end{document}